\documentclass[letterpaper]{article} 
\usepackage{aaai24}  
\usepackage{times}  
\usepackage{helvet}  
\usepackage{courier}  
\usepackage[hyphens]{url}  
\usepackage{graphicx} 
\usepackage{natbib}  
\usepackage{caption} 
\usepackage{algorithm}
\usepackage{algorithmic}
\usepackage{enumitem}

\usepackage{newfloat}
\usepackage{listings}
\DeclareCaptionStyle{ruled}{labelfont=normalfont,labelsep=colon,strut=off} 
\floatstyle{ruled}
\newfloat{listing}{tb}{lst}{}
\floatname{listing}{Listing}
\title{Discerning Temporal Difference Learning}
\author{
    Jianfei Ma
}
\affiliations{
    Northwestern Polytechnical University\\
    School of Mathematics and Statistics\\
    matrixfeeney@gmail.com
}

\usepackage{amsthm}
\usepackage{amsfonts}
\usepackage{amssymb}
\usepackage{amsmath}

\DeclareMathOperator*{\argmin}{arg\,min}

\newcommand{\be}{\mathbf{e}}
\newcommand{\bPhi}{\boldsymbol{\Phi}}
\newcommand{\bt}{\boldsymbol{\theta}}
\newcommand{\bPi}{\boldsymbol{\Pi}}
\newcommand{\bI}{\mathbf{I}}
\newcommand{\bPpi}{\mathbf{P}_{\pi}}
\newcommand{\brpi}{\mathbf{r}_{\pi}}
\newcommand{\bF}{\mathbf{F}}

\newcommand{\bdp}{\mathbf{d}_{\pi}}
\newcommand{\bD}{\mathbf{D}}

\newcommand{\bL}{\boldsymbol{\Lambda}}
\newcommand{\bb}{\mathbf{b}}
\newcommand{\bA}{\mathbf{A}}
\newcommand{\bV}{\mathbf{V}}
\newcommand{\bdpi}{\mathbf{d}_{\pi}}
\newcommand{\bone}{\mathbf{1}}

\newcounter{thm_counter}
\newcounter{lem_counter}
\newcounter{pro_counter}
\newcounter{cor_counter}
\newcounter{ass_counter}
\newcounter{def_counter}
\newcounter{rmk_counter}

\newtheorem{theorem}[thm_counter]{Theorem}
\newtheorem{proposition}[pro_counter]{Proposition}
\newtheorem{lemma}[lem_counter]{Lemma}
\newtheorem{corollary}[cor_counter]{Corollary}
\newtheorem{assumption}[ass_counter]{Assumption}
\newtheorem{definition}[def_counter]{Definition}
\newtheorem{remark}[rmk_counter]{Remark}

\begin{document}

\maketitle

\begin{abstract}
  Temporal difference learning (TD) is a foundational concept in reinforcement learning (RL), aimed at efficiently assessing a policy's value function. TD($\lambda$), a potent variant, incorporates a memory trace to distribute the prediction error into the historical context. However, this approach often neglects the significance of historical states and the relative importance of propagating the TD error, influenced by challenges such as visitation imbalance or outcome noise. To address this, we propose a novel TD algorithm named discerning TD learning (DTD), which allows flexible emphasis functions—predetermined or adapted during training—to allocate efforts effectively across states. We establish the convergence properties of our method within a specific class of emphasis functions and showcase its promising potential for adaptation to deep RL contexts. Empirical results underscore that employing a judicious emphasis function not only improves value estimation but also expedites learning across diverse scenarios.
\end{abstract}

\section{Introduction}
In reinforcement learning, efficiently predicting future rewards based on past experiences is a fundamental challenge. TD(0) assigns credit using the difference between successive predictions \cite{DBLP:journals/ml/Sutton88}, offering online and recursive capabilities, but its scope is limited to the current observed state. On the other hand, TD($\lambda$) \cite{Sutton1998}, when combined with the eligibility trace, assigns credit to all historical states, using a recency heuristic to propagate credit. However, it lacks consideration for the relative importance of each state, that is, how much emphasis should be given to each historical state to make better predictions? Both approaches have found success in modern RL algorithms \cite{DBLP:journals/corr/MnihKSGAWR13} \cite{DBLP:journals/corr/SchulmanWDRK17}, but they uniformly weigh states, overlooking the potential benefits of emphasis-aware credit assignment.

In various circumstances, the emphasis-awareness becomes crucial. The accurate estimation of a value function often faces inherent challenges, including visitation imbalances and noisy outcomes, particularly in reward-seeking tasks. Due to variations in initial state distributions or transition models, agents tend to update high-frequency states more frequently, while less attention is given to less frequent or near-terminating states, such as goal states. This phenomenon is especially prevalent in episodic tasks with eligibility traces, where more frequent states are updated each time a new state is encountered while states close to termination have shorter trace positions, resulting in fewer updates. In such cases, diverging update frequencies can lead to imbalanced value estimations. Furthermore, the added complexity arising from noisy observations \cite{DBLP:journals/corr/abs-2202-09699} or rewards \cite{DBLP:conf/aaai/WangLL20} exacerbates the estimation challenge. Injected noise can lead to erroneous predictions, propagating inaccuracies to other states. Moreover, in reality, the most rewarding states are often rare occurrences, particularly in situations where most states have sparse rewards. The fundamental idea is to emphasize less frequent or more valuable states by increasing individual update magnitudes or reducing attention on states with adverse factors, using a nonnegative emphasis function. While existing approaches \cite{DBLP:journals/jmlr/SuttonMW16} \cite{DBLP:conf/icml/AnandP21} \cite{DBLP:journals/corr/abs-2202-09699} provide some insights, they either lack a thorough investigation of emphasis functions in diverse scenarios or overlook the mutual influence of the emphasis between states.

In this paper, we introduce a novel class of TD learning methods based on a fundamental identity. This identity, when viewed forward, directly incorporates an emphasis function that prioritizes various multi-step return combinations, offering enhanced flexibility. We establish a connection to a computationally efficient backward view that updates online, with its structure revealing the emphasis function's role. We provide theoretical analysis demonstrating that for a specific class of emphasis functions, our method converges to the optimal solution in the sense of an emphasized objective. Illustrative examples are presented to investigate emphasis choices in diverse scenarios. These examples reveal that our proposed approach, DTD, can enhance value estimation and expedite learning, with particularly noticeable benefits from more compact choices like the absolute expected TD error. Moreover, the newly developed type of return function holds promise for adaptation to DRL scenarios, especially where accurate advantage estimation is crucial \cite{DBLP:journals/corr/SchulmanMLJA15}. Additionally, we establish a connection to prioritized sampling \cite{DBLP:journals/corr/SchaulQAS15} in cases where the data is Markovian. Lastly, we initiate a discussion on the design of the emphasis function from various perspectives.

\section{Preliminaries}
\label{sec:preliminaries}
\subsection{Notation}
\label{sec:notation}
Let's denote $\|\cdot\|_{\bL}$ the vector norm induced by a positive definite matrix $\bL$, i.e. $\|x\|_{\bL} = \sqrt{x^{\top} \bL x}$. And the corresponding induced matrix norm is $\| \bA \|_{\bL} = \max_{\| x \|_{\bL} = 1} \| \bA x \|_{\bL}$. With $\bL = \bI$, it comes to the Euclidean-induced norm, for which we drop the subscript as $\| \bA \|$. For simplicity, $\bone$ denotes the all-one vector. We indicate random variables by capital letters (e.g $S_{t}, A_{t}$), realization by lowercase letters (e.g $s_{t}, a_{t}$).

\subsection{Problem Setting}
\label{sec:problem-setting}
Consider an infinite-horizon discounted MDP, defined by a tuple $(\mathcal{S}, \mathcal{A}, P, r, \rho_{0}, \gamma)$, with a finite state space $\mathcal{S}$, a finite action space $\mathcal{A}$, a transition kernel $P: \mathcal{S} \times \mathcal{A} \times \mathcal{S} \rightarrow \mathbb{R}$, a reward function $r: \mathcal{S} \times \mathcal{A} \rightarrow \mathbb{R}$, an initial state distribution $\rho_{0}: \mathcal{S} \rightarrow \mathbb{R}$, and a discount factor $\gamma \in [0, 1)$. Being at a state $s_{t} \in \mathcal{S}$, the agent takes an action $a_{t} \in \mathcal{A}$ according to some policy $\pi$, which assigns a probability $\pi(a_{t} | s_{t})$ to the choice. After the environment receives $a_{t}$, it emits a reward $r_{t}$, and sends the agent to a new state $s_{t+1} \sim P(s_{t+1} | s_{t}, a_{t})$. Repeating this procedure, the discounted return can be fulfilled as $G_{t} = \sum\limits_{t=0}^{\infty}\gamma^{t}R_{t}$. We denote $\bPpi \in \mathbb{R}^{|\mathcal{S}| \times |\mathcal{S}|}$ the state transition matrix and $\brpi \in \mathbb{R}^{|\mathcal{S}|}$ the expected immediate reward vector. And the steady-state distribution is denoted as $\bdpi(s)$, which we assume exists and is positive at all states. Let $\bD$ denote the diagonal matrix with $\bdpi$ on its diagonal. The prediction problem we are interested in is to estimate the value function:
\begin{equation}
  \label{eq:1}
  v_{\pi}(s) = \mathbb{E}_{\pi}[G_{t} | S_{t} = s].
\end{equation}
When the state space is large or even continuous, it is beneficial to use function approximation $\hat{v}(s, \bt)$ to represent $v$ to generalize across states. In particular, if the feature is expressive, it is convenient to use linear function approximation:
\begin{equation}
  \label{eq:2}
  \hat{v}(s, \bt) = \phi(s)^{\top} \bt,
\end{equation}
where $\phi(s)$ is the feature vector at state $s$. With each feature vector of length $K$ being at the row of the matrix $\bPhi$, we can compactly represent the value function vector as $\bV_{\bt} = \bPhi \bt$. For any value function $\bV$, the most representable solution in the span of $\bPhi$ corresponds to \cite{DBLP:conf/icml/SuttonMPBSSW09} \cite{DBLP:journals/tac/YuB09}:
\begin{equation}
  \label{eq:3}
  \bPi \bV = \bPhi \bt^{\star}\ \text{where}\ \bt^{\star} = \argmin_{\bt} \| \bPhi \bt - \bV\|_{\bD},
\end{equation}
where $\bPi$ is the projection matrix in the form of:
\begin{equation}
  \label{eq:4}
  \bPi = \bPhi (\bPhi^{\top} \bD \bPhi)^{-1} \bPhi^{\top} \bD.
\end{equation}
To solve Eq. \eqref{eq:3}, simulation-based approaches are often utilized. With $\bV$ equal to the one-step TD target, TD(0) performs stochastic gradient descent to minimize the TD error:
\begin{equation}
  \label{eq:5}
  \bt_{t + 1} = \bt_{t} + \alpha_{t} \delta_{t} \phi_{t},
\end{equation}
where
\begin{equation}
  \label{eq:6}
  \delta_{t} = R_{t + 1} + \gamma \bt^{\top}_{t} \phi_{t + 1} - \bt^{\top}_{t} \phi_{t}
\end{equation}
is the TD error, and $\alpha_{t}$ is the learning rate. The advantage of this approach is that it incrementally updates the weight vector at every time step, without requiring waiting until the end of an episode. However, it only takes effect on the current observed state. TD($\lambda$), on the other hand, while sustains the same benefit of the online update, it is able to influence past experiences. Those past experiences can be viewed as eligible experiences that receive credit from the latest experience. This results in a more efficient update:
\begin{equation}
  \begin{aligned}
      \label{eq:7}
  \be_{t} & = \gamma \lambda \be_{t - 1} + \phi_{t} \\
  \bt_{t + 1} & = \bt_{t} + \alpha_{t} \be_{t} \delta_{t},
  \end{aligned}
\end{equation}
where $\be_{t}$ is called \emph{eligibility trace}, with $\be_{-1} = 0$. It is this temporally extended memory that allows the TD error at the current time step to be propagated to the states along the path that leads to the current state.

While the additional parameter $\lambda \in [0, 1]$ is seamlessly integrated into the trace, it originates from the conventional forward view that directly interpolates $n$-step return exponentially forming the $\lambda$-return:
\begin{equation}
  \label{eq:8}
    G_{t}^{\lambda} = (1 - \lambda) \sum\limits_{n=1}^{\infty}\lambda^{n - 1} G_{t}^{(n)},
\end{equation}
where
\begin{equation}
  \label{eq:9}
  G_{t}^{(n)} = \sum\limits_{k = 1}^{n} \gamma^{k - 1}R_{t + k} + \gamma^{n}\hat{v}(S_{t + n}, \bt),
\end{equation}
is the $n$-step return. The method based on the target $G_{t}^{\lambda}$ is called the $\lambda$-return algorithm, which has been proven to achieve the same weight updates as offline TD($\lambda$) \cite{DBLP:journals/ml/Sutton88} \cite{Sutton1998}.

\section{Discerning Temporal Difference Learning}
\label{sec:disc-temp-diff}
The limitation of TD($\lambda$) is that it fails to account for the importance of each historical state or to consider the relative significance of propagating the TD error. To address this issue, we derive our new return function that directly incorporates emphasis start based on an important identity. Consider any function $f: \mathcal{S} \rightarrow \mathbb{R}$, the following identity holds:
\begin{equation}
  \label{eq:10}
  \sum\limits_{n = 0}^{\infty} \lambda^{n} (f_{t + n} - f_{t + n + 1} \lambda) = f_{t},
\end{equation}
which generalizes the multiplier $1 - \lambda$ in the $\lambda$-return. An interesting property is that the above holds for any real-valued function. However, such a function class would be too large to accommodate our purpose. We, therefore, constrain it into the bounded positive real-valued function as the emphasis function, measuring the significance of each state, with which we can derive a new return function as follows:
\begin{proposition}
  \label{prop:1}
  For any $f: \mathcal{S} \rightarrow \mathbb{R}^{+}$, it holds that:
\begin{align}
  \bar{G}_{t}^{\lambda, f} & \dot{=} \sum\limits_{n=1}^{\infty}\lambda^{n - 1} (f_{t + n - 1} - f_{t + n} \lambda) G_{t}^{(n)} \label{eq:11}\\
  & = f_{t}\hat{v}(S_{t + n}, \bt) + \sum\limits_{n=0}^{\infty}(\gamma\lambda)^{n} \delta_{t + n} f_{t + n} \label{eq:12}.
\end{align}
\end{proposition}
We defer the precise proof to the Appendix.

Intuitively, as Eq. \ref{eq:12} indicates, each TD error term is reweighted by the emphasis function so as to control the relative strength of each future state. $\bar{G}_{t}^{\lambda, f}$ here nonetheless is unnormalized unless with a scalar multiplier $\frac{1}{f_{t}}$. Henceforth, we will reload the notation as $G_{t}^{\lambda, f} \dot{=} \frac{1}{f_{t}}\bar{G}_{t}^{\lambda, f}$, named as discerning $\lambda$-return.

Next, we will formally deliver DTD with a composition of an emphasized objective and the discerning $\lambda$-return as mentioned above. Denote $\bF$ as a diagonal matrix with $f$ on its diagonal, furthermore $\bL = \bF \bD \bF$, we therefore minimize an emphasized objective analogous to Eq. $\ref{eq:3}$:
\begin{equation}
  \label{eq:13}
  \bt^{\star} = \argmin_{\bt} \| \bPhi \bt - \bV\|_{\bL},
\end{equation}
which modulates the steady state probability with the square of the emphasis function. The projection matrix can be expressed as:
\begin{equation}
  \label{eq:14}
  \bPi^{f} = \bPhi (\bPhi^{\top} \bL \bPhi)^{-1} \bPhi^{\top} \bL.
\end{equation}  
Any solution to Eq. \ref{eq:13} will have an orthogonal difference to the emphasized basis such that:
\begin{equation}
  \label{eq:orthogonal}
  \bV - \bPi^{f} \bV \perp \bL \bPhi.
\end{equation}
By combining the discerning $\lambda$-return as the target $\bV$, and manipulating the equivalence to the backward view similar to the deduction of the TD$(\lambda)$, we can derive the DTD($\lambda$) update:
\begin{equation}
  \label{eq:15}
  \begin{aligned}
  \be_{t} & = \gamma \lambda \be_{t - 1} + f_{t}\phi_{t} \\
  \bt_{t + 1} & = \bt_{t} + \alpha_{t} \be_{t} \delta_{t} f_{t} \\
  \be_{-1} & = 0,
  \end{aligned}
\end{equation}
which distinguishes the historical state as well as regulates the relative importance of propagating the TD error. The complete algorithm is outlined in Alg. \ref{alg:DTD} with a general function approximator.
\begin{algorithm}[h]
\caption{DTD($\lambda$)}
\label{alg:DTD}
\textbf{Input}: $\pi, v_{\bt}, f, \gamma, \lambda$\\
\textbf{Initialize}: $\bt$ arbitrarily
\begin{algorithmic}[1] 
  \FOR{each episode}
  \STATE Initialize $S$
  \STATE Initialize $\be$
  \REPEAT 
  \STATE Take action $A \sim \pi(\cdot | S)$, observe $R, S'$
  \STATE $\be \gets \gamma \lambda \be + f(S) \nabla \hat{v}(S, \bt)$
  \STATE $\delta \gets R + \gamma \hat{v}(S', \bt) - \hat{v}(S, \bt)$
  \STATE $\bt \gets \bt + \alpha_{t} \delta \be f(S)$
  \STATE $S \gets S'$
  \UNTIL{$S$ is terminal}
  \ENDFOR
  \STATE \textbf{return} $\bt$
\end{algorithmic}
\end{algorithm}
\section{Theoretical Analysis}
\label{sec:convergence-dtd}
We embark on a theoretical exploration of the algorithm's convergence behavior concerning the emphasis function, offering analyses for both parameter-independent and parameter-dependent scenarios. To establish the foundation, we introduce several necessary assumptions:
\begin{assumption}
  \label{asp:1}
  The Markov chain $\{S\}$ is irreducible and aperiodic.
\end{assumption}
\begin{assumption}
  \label{asp:2}  
  $\bPhi$ has linearly independent columns.
\end{assumption}
\begin{assumption}
  \label{asp:3}  
  The learning rate $\{\alpha_{t}\}$ is non-increasing, and satisfies Robbins-Monro conditions:
  \begin{equation}
    \label{eq:16}
    \sum\limits_{t=0}^{\infty}\alpha_{t} = \infty\quad \text{and}\quad \sum\limits_{t=0}^{\infty}\alpha_{t}^{2} < \infty.
  \end{equation}
\end{assumption}
\begin{assumption}
  \label{asp:4}
  $f_{t + n - 1} - f_{t + n} \lambda$ is independent of $G_{t}^{(n)}$ for $n \in \mathbb{N}^{+}$.
\end{assumption}
Assumptions \ref{asp:1}--\ref{asp:3} adhere to the standard framework for analyzing linear TD methods (see, for instance, \cite{DBLP:journals/tac/TsitsiklisR97} \cite{DBLP:conf/icml/Yu10}). Assumption \ref{asp:4} is introduced to facilitate analytical operator analysis.

To characterize the discerning $\lambda$-return in its expected behavior, we introduce a notion of the DTD($\lambda$) operator, which encapsulates the essence of the forward-view DTD($\lambda$):
\begin{definition}
  Discerning $\lambda$-return operator:
\begin{equation}
  \label{eq:17}
  \begin{aligned}
    \mathcal{T}^{\lambda, f}(\bV_{\bt}) = \bF^{-1} &\sum\limits_{n = 0}^{\infty} \lambda^{n}  (\bPpi^{n} (\bI - \lambda \bPpi) \bF) \bone \circ \\
    \Bigl(&\sum\limits_{t=0}^{n}(\gamma \bPpi)^{t} \brpi + (\gamma \bPpi)^{n + 1} \bV_{\bt}  \Bigr), 
  \end{aligned}
\end{equation}
where $\circ$ is the Hadamard product between matrices.
\end{definition}
Next, we examine the contraction condition about $\mathcal{T}^{\lambda, f}$.
\begin{theorem}
  \label{thm:contraction}
  Let $\sigma_{\text{min}}(\bF)$ represent the smallest singular value of matrix $\bF$. The mapping $\mathcal{T}^{\lambda, f}$ is a contraction for the parameter-independent case if it satisfies:
  \begin{enumerate}[label=\roman*)]    
  \item $\| \bF \|_{\bL} < \frac{\sigma_{\text{min}}(\bF) (1 - \gamma \lambda)}{\gamma \| \bone \|_{\bL} \| \bI - \lambda \bPpi \|_{\bL}}$.
  \item For the parameter-dependent case, if further there exists a Lipschitz constant $\kappa \in \bigl[0, \frac{(1 - \lambda)(1 - \gamma) \sigma_{\text{min}}(\bF)}{r_{\text{max}} \| \bone \|_{\bL} \| \bI - \lambda \bPpi \|_{\bL}}\bigr)$ such that for any $\bF_{1}(\bV_{\bt_{1}}), \bF_{2}(\bV_{\bt_{2}})$:
  \begin{equation*}
    \label{eq:18}
    \| \bF_{1} - \bF_{2} \|_{\bL} \leq \frac{\kappa}{2} \| \bV_{\bt_{1}} - \bV_{\bt_{2}} \|_{\bL},
  \end{equation*}
  then it is a contraction mapping provided that:\\
  $\| \bF \|_{\bL} < \frac{\sigma_{\text{min}}(\bF) (1 - \gamma \lambda)}{\gamma \| \bone \|_{\bL} \| \bI - \lambda \bPpi \|_{\bL}} - \frac{(1 - \gamma \lambda) r_{\text{max}} \kappa}{\gamma (1 - \lambda) (1 - \gamma)}, \forall \bt \in \boldsymbol{\Theta}$,
  \end{enumerate}
\end{theorem}
where $\boldsymbol{\Theta}$ is the parameter space that can be a suitable subset of $\mathbb{R}^{K}$.

This property guarantees the uniqueness of the fixed point. To avoid repetition, we define the function class $\Xi$ as a set that satisfies either condition $i)$ or $ii)$.
\begin{remark}
  \label{rm:1}
  Note $\| \bone \|_{\bL}$ is simply the expected value of the squared emphasis function under the steady-state distribution, i.e. $\mathbb{E}_{\bdp}[f^{2}(S)]$. In practice, if we can scale the emphasis function into a considerably small range (i.e. $[0, 1]$), we will have a broader spectrum that enhances the contraction.
\end{remark}
\begin{corollary}
  \label{cor:1}
  $\bPi^{f}\mathcal{T}^{\lambda, f}$ is a contraction mapping for any $f \in \Xi$.
\end{corollary}
\begin{proof}
  From Eq. \ref{eq:orthogonal} we know that the difference between $\bV$ and $\bPi^{f} \bV$ is orthogonal to the $\bPhi$ in the sense of the $\| \cdot \|_{\bL}$, whereas $\bPi^{f} \bV$ is a linear combination of $\bPhi$, therefore $\bV - \bPi^{f} \bV \perp \bL \bPi^{f} \bV$. By Pythagorean theorem, it follows that $\bPi^{f}$ is non-expansive. Since $\mathcal{T}^{\lambda, f}$ is a contraction mapping, thus the composition is also a contraction mapping.
\end{proof}
Consider a process $X_{t} = \{S_{t}, S_{t + 1}, \be_{t}\}$, which is a finite Markov process as $\be_{t}$ is only dependent up to $S_{t}$. Thereby, the update in Eq. \ref{eq:15} can be simplified as follows:
\begin{equation}
  \label{eq:19}
  \bt_{t + 1} = \bt_{t} + \alpha_{t} \left(A(X_{t}) \bt_{t} + b(X_{t}) \right),
\end{equation}
where $A(X_{t}) = \be_{t}(\gamma \phi(S_{t + 1}) - \phi(S_{t}))^{\top} f(S_{t})$ and $b(X_{t}) = \be_{t} R_{t} f(S_{t})$. It was shown that this update exhibits asymptotic behavior akin to a deterministic variant in the sense of the steady state distribution \cite{DBLP:books/sp/BenvenisteMP90} \cite{DBLP:journals/tac/TsitsiklisR97}. As a result, we now delve into the essential quantities required to represent such a variant. Denoting $\bA = \mathbb{E}_{\bdp}[A(X_{t})]$ and $\bb = \mathbb{E}_{\bdp}[b(X_{t})]$, they can be succinctly expressed in the matrix form:
\begin{lemma}
  \label{lm:1}
  Denote $\bar{\bF} = \lim_{t \rightarrow \infty} \bF_{t}$, which is assumed to exist, then
  \begin{equation}
    \label{eq:20}
    \begin{aligned}
    \bA & = \bPhi^{\top} \bar{\bF} \bD (\bI - \gamma \lambda \bPpi)^{-1} \bar{\bF} (\gamma \bPpi - \bI) \bPhi \\
    \bb & = \bPhi^{\top} \bar{\bF} \bD (\bI - \gamma \lambda \bPpi)^{-1} \bar{\bF} \brpi.
    \end{aligned}
  \end{equation}
  In the parameter-independent case, $\bF_{t} \equiv \bF$, while in the parameter-dependent case, $\bF_{t} = \bF(\bV_{\bt_{t}})$.  
\end{lemma}
Using those quantities, we can express the deterministic variant as follows:
\begin{equation}
  \label{eq:21}
  \bt_{t + 1} = \bt_{t} + \alpha_{t} \left(\bA \bt_{t} + \bb\right).
\end{equation}
To establish a connection with the earlier contraction results, it can be shown that:
\begin{equation}
  \label{eq:22}
  \bA \bt + \bb = \bPhi^{\top} \bL (\mathcal{T}^{\lambda, f}(\bPhi \bt) - \bPhi \bt).
\end{equation}
Building upon this result and the established contraction condition, we can derive a fundamental component for the convergence:
\begin{lemma}
  \label{lm:2}
  $\bA$ is negative definite for any $f \in \Xi$.
\end{lemma}
With the above results, we can now demonstrate the convergence result:
\begin{theorem}
  \label{thm:convergence}
  The updates induced by DTD($\lambda$) converge to a unique fixed point $\bt^{\star}$ satisfying $\bA \bt^{\star} + \bb = 0$ for any $f \in \Xi$.
\end{theorem}
\section{Experiments}
\label{sec:empirical-results}
In this section, we delve into the impact of DTD($\lambda$)'s emphasizing effect, whether the emphasis function is predetermined or adapted during training. We examine scenarios involving visitation imbalance or noisy outcomes to determine if DTD($\lambda$) can address these challenges and enhance overall performance using predetermined emphasis. Regarding adaptive emphasis, we explore a more compact form, namely the absolute expected TD error, to assess the influence of non-stationary emphasis on the prediction tasks. Our findings demonstrate that, firstly, the update-rebalancing and noise-averse effects effectively handle inherent prediction difficulties; secondly, the promising adaptive emphasis surpasses numerous baselines across diverse tasks.
\subsection{Evaluation}
We choose the mean-square projected Bellman error (MSPBE) \cite{DBLP:conf/icml/SuttonMPBSSW09} as the performance metric, as it quantifies the deviation from the most representative functions attainable with the given features of $\mathcal{T} \bV_{\bt}$, where $\mathcal{T}$ is the Bellman operator. The MSPBE is defined as follows:
\begin{equation}
  \label{eq:23}
  \text{MSPBE}(\bt) = \| \bV_{\bt} - \bPi \mathcal{T} \bV_{\bt}\|_{\bD}.
\end{equation}
The experiments are carried out over 50 independent runs, each spanning 5000 environment steps. The depicted curves report the best performance with extensive parameter sweeping. They present the aggregated mean along with error bars representing the standard deviation. All of the problems are episodic, undiscounted, and involve only a fixed target policy.
\subsection{More or Less}
\label{sec:more-less}
To illustrate the impact of the emphasis in scenarios with the visitation imbalance, we examine a 5-state random-walk problem featuring three distinct initial state distributions. In each episode, the agent starts deterministically from either the leftmost, middle, or rightmost state. This selection of the initial state leads to varying visitation frequencies among the neighboring states. States near the initial choice are more frequently visited, while those farther away are less likely to be encountered. Consequently, these three initial state distributions result in overall state visitation frequencies that exhibit left-skewed, center-elevated, or right-skewed patterns. The chain includes two terminal states located at opposite ends, with all transitions uniformly distributed across states. Rewards are uniformly zero, except when transitioning into the right terminal state. Tabular features represent the state characteristics. The challenge with TD($\lambda$) arises from the tendency to update more frequently visited states heavily, while paying less attention to states that are visited infrequently. This discrepancy emerges due to the higher occurrence of more frequent states in the eligibility trace, resulting in more updates from ahead-time steps. This effect can be even more amplified if these states persist in the trace for an extended duration. Conversely, states that are infrequent visitors or have shorter durations in the trace undergo fewer updates. To rectify these imbalanced updates, our approach, DTD($\lambda$), addresses the shortage of total updates by increasing the update magnitude for infrequent states. The emphasis function we tailor is based on the inverse of the normalized empirical state visitation counts, which are then scaled to lie within the range of [0, 1] as recommended by Remark \ref{rm:1}. Finally, we take the square root to restore the original quantity.

\begin{figure*}[t]
\centering
\includegraphics[width=0.8\textwidth]{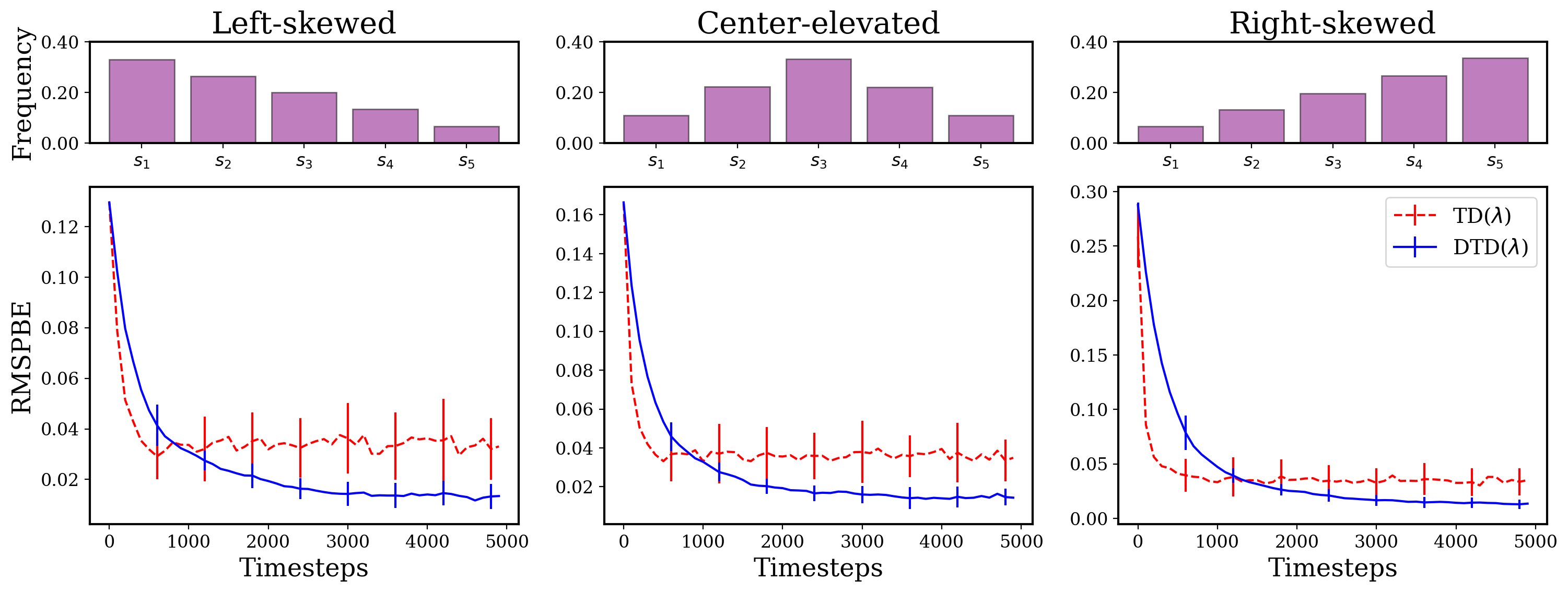}
\caption{Top: State visitation frequency for different states; Bottom: Learning curve of MSPBE for algorithmic comparison. The three tasks are based on three different initial distributions.}
\label{fig:1}
\end{figure*}

To demonstrate the efficacy of our method for combatting the noisy outcome in scenarios with the perturbed reward, we consider a larger problem with 10 states where transitions have a uniform reward level with added noises. In order to isolate the influence of visitation imbalance, the initial state is selected uniformly from all available states, and all transitions are executed uniformly. The noise is symmetric for the transition from a state but varies across states. We consider the noises as $\sigma = [0,1, 0.2, \dots, 1]$ for states $s_{1}, s_{2}, \dots, s_{10}$. Three different reward levels $r = [-1, 0, 1]$ are tested with varying difficulties. The reward that the agent actually receives is the reward level with an added Gaussian noise $\mathcal{N}(0, \sigma(s_{i}))$ for transition from $s_{i}$. Incorrect predictions can occur when the agent is unaware of underlying noises, leading it astray from the true value. Moreover, this can lead to even more severe consequences, as the erroneous TD error may propagate to other states. DTD($\lambda$) offers greater flexibility in addressing this situation through the emphasis function, which places resistance on states with high noise levels while prioritizing those with low noise levels. To that end, we introduce a prior into the design of the emphasis function, specifically the negative exponential of the noise levels, denoted as $\exp(-\sigma(s))$, to mitigate the influence of unpredictable outcomes. It is also scaled to lie within the range of [0, 1] and applied the square root.

\begin{figure}[h]
\centering
\includegraphics[width=0.8\columnwidth]{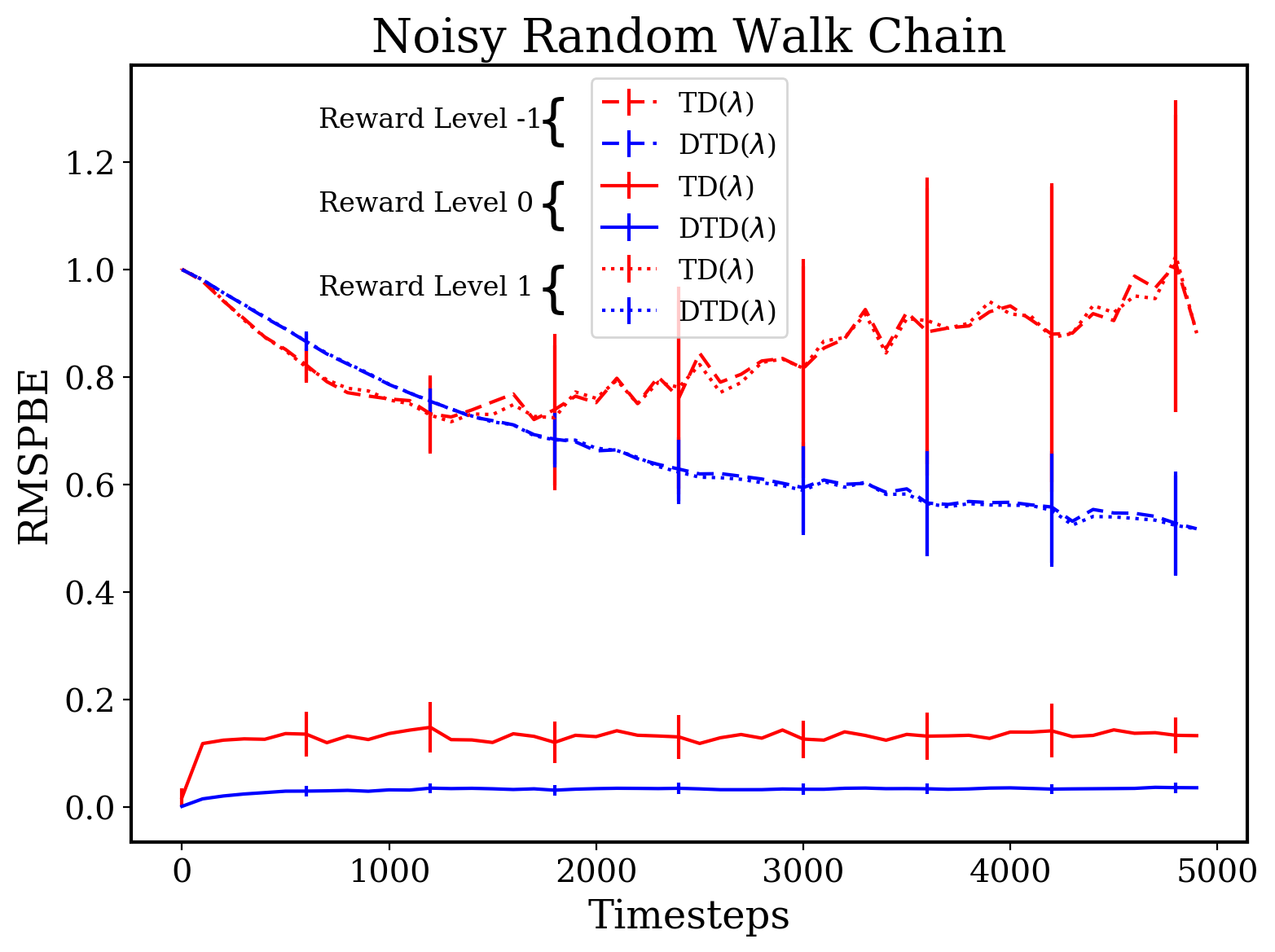}
\caption{Learning curve of MSPBE of different reward levels with added noises.}
\label{fig:2}
\end{figure}

The results depicted in Fig. \ref{fig:1} indicate that regardless of the skewness of the state visitation frequency, DTD($\lambda$) effectively rebalances updates to achieve improved overall predictions. While TD($\lambda$) shows faster progress in the early stages, the aliasing effect of TD error and the lack of updates for infrequent states become more pronounced, leading to its struggle in further reducing the error. In contrast, DTD($\lambda$) allocates more attention to those infrequent states, resulting in a more balanced update process. In the case illustrated in Fig. \ref{fig:2}, even at a zero reward level, TD($\lambda$) results in a larger prediction error with higher variation, while DTD($\lambda$) consistently maintains a relatively small prediction error with less variability. As the reward level becomes non-zero, the increased complexity involved in predicting the true value causes TD($\lambda$) to progressively deviate from its initial prediction. In contrast, DTD($\lambda$) effectively discerns different noise levels, leading to a reduction in prediction errors. From the learning curve, We hypothesize that with an increased computational budget, DTD($\lambda$) can yield a much lower prediction error.

The idea of allocating attention selectively can be enlightening. In reality, valuable states are often infrequently encountered, and achieving a goal can require substantial effort. By focusing more attention on these crucial outcomes, we can enhance the influence of pathways leading to them.

\subsection{Adaptive Emphasis}
\label{sec:non-stat-emph}
\begin{figure*}[t]
\centering
\includegraphics[width=1\textwidth]{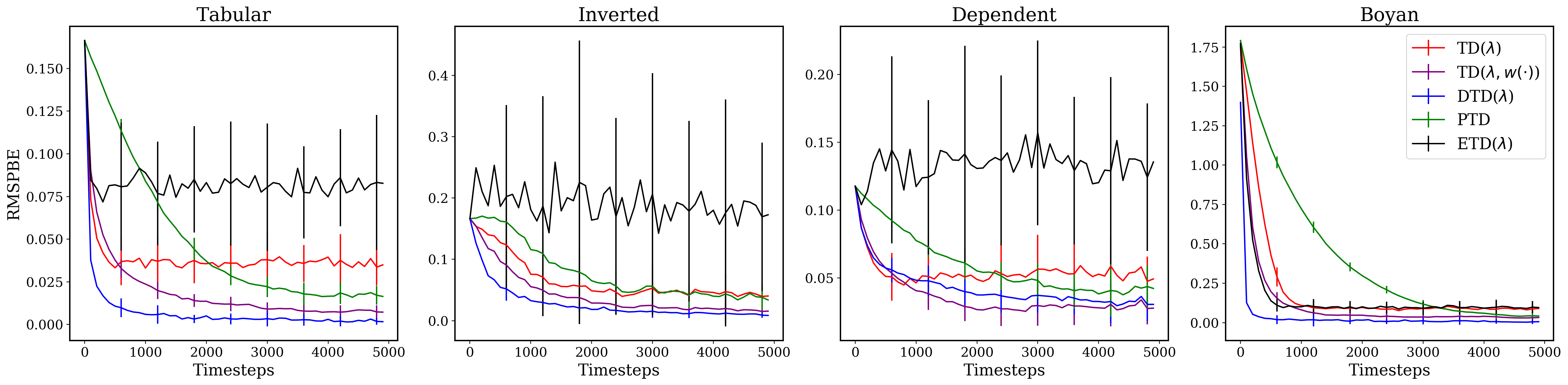}
\caption{Learning curve of MSPBE on 5-state random walk chain with tabular, inverted, and dependent feature representation and the 13-state Boyan chain. Baselines are chosen to be emphatic and with selective updating.}
\label{fig:3}
\end{figure*}
the predetermined emphasis can vary depending on the specific problems, making manual crafting challenging. Is it possible to devise a compact emphasis that directly aligns with the nature of the prediction task? In this part, we examine the parameter-dependent emphasis, namely the absolute expected TD error, evaluated using the true dynamics, to showcase the effectiveness of the prediction-oriented emphasis for accelerating learning. We investigate four additional tasks, three of which share the same setup as the 5-state problem discussed earlier, but the initial state distribution is set to the middle state by default. There involves three representations as introduced in \cite{DBLP:conf/icml/SuttonMPBSSW09}: tabular, inverted (inappropriate state generalization), and dependent (insufficient representation), posing aliasing and representation challenges that standard methods are difficult to solve. Due to space limits, we refer the reader to \cite{DBLP:conf/icml/SuttonMPBSSW09} for more detailed descriptions. The last task is a 13-state Boyan chain with 4 features \cite{DBLP:journals/ml/Boyan02}, which serves as a standard benchmark for evaluating TD-style algorithms. In addition to comparing DTD($\lambda$) with TD($\lambda$), we assess its performance against several baselines that incorporate varying levels of emphasis. These baselines include an on-policy emphatic variant of ETD($\lambda$) \cite{DBLP:journals/jmlr/SuttonMW16}, the preferential approach PTD \cite{DBLP:conf/icml/AnandP21}, and the selective updating TD($\lambda, w(\cdot)$) \cite{DBLP:journals/corr/abs-2202-09699}.

The results presented in Fig. \ref{fig:3} demonstrate that DTD($\lambda$) outperforms the other methods across the majority of tasks, exhibiting both rapid initial learning and minimal variability. Even under challenging representations, it can not only mitigate incorrect state aliasing where the update from one state only changes the parameters of other states, but also manifest the benefit of adaptive updating for limited capacity where the span of the feature space is insufficient to solve the task exactly. It is worth noting that ETD($\lambda$) may suffer from the high variance issue of the follow-on trace, which could explain its poor performance. On the other hand, PTD appears to interpolate between TD(0) update and no update, causing its updates to be centered around TD(0) and possibly missing out on the advantages of combining different $n$-step returns. The approach most closely related to ours is TD($\lambda, w(\cdot)$), which employs a similar eligibility trace. However, it does not take into account the importance of propagating the TD error. Comparing our approach to TD($\lambda, w(\cdot)$) is essentially a direct test of the significance of our emphasis factor multiplied by the TD error. The results clearly show that removing this emphasis factor significantly degrades the performance, underscoring its crucial role in amplifying the propagation of TD error and its relative influence to historical states when combined with the eligibility trace.

\section{Extendibility for DRL}
In this section, we delve deeper into the aspects of DTD, particularly its connection to advantage estimation and its relevance to prioritized sampling. The former is closely linked to the concept of discerning $\lambda$-returns, which holds the potential for further enhancing variance reduction. The latter aspect establishes a relationship with non-uniform sampling, wherein DTD(0) can yield a similar prioritization effect.
\label{sec:extending-drl}
\subsection{Discerning Advantage Estimator}
In the realm of DRL algorithms, the variance of policy gradients often becomes a bottleneck for overall performance, particularly in on-policy algorithms \cite{DBLP:conf/icml/SchulmanLAJM15} \cite{DBLP:journals/corr/SchulmanWDRK17}. Generalized Advantage Estimator (GAE) \cite{DBLP:journals/corr/SchulmanMLJA15} offers an effective approach to mitigate the high variance stemming from lengthy trajectory estimates. Notably, \cite{DBLP:journals/tog/PengALP18} and \cite{DBLP:journals/corr/abs-2302-00533} establish a close connection between GAE and the $\lambda$-return, albeit with a baseline function integrated to reduce the variance. Similarly, we can derive the Discerning Advantage Estimator (DAE), in the context of the discerning $\lambda$-return:
\label{sec:disc-advant-estim}
\begin{definition}
  \label{def:DAE}
  Discerning Advantage Estimator:
\begin{equation}
  \label{eqq:13}
  \begin{aligned}
    \hat{A}_{t}^{\text{DAE}(\lambda, \gamma, f)} = \frac{1}{f_{t}} \sum\limits_{n=0}^{\infty}(\gamma\lambda)^{n} \delta_{t + n} f_{t + n}.
  \end{aligned}
\end{equation}  
\end{definition}
Aside from reducing variance with the baseline, DAE incorporates emphasis to reweight the TD error terms. We hypothesize that a well-chosen emphasis function can additionally lower the variance of the advantage estimate, specifically by quantifying the variance of the $n$-step return.
\subsection{Connection to Prioritized Sampling}
\label{sec:conn-prior-sampl}
Prioritized experience replay (PER) \cite{DBLP:journals/corr/SchaulQAS15} highlights that the ``uniformness" of experience replay \cite{DBLP:journals/corr/MnihKSGAWR13} adheres to the same frequency as the original experiences, yet it fails to account for the significance of individual samples. This method updates the value function by assigning a priority to each sample, thus giving precedence to samples with higher significance, specifically those proportional to the sampled absolute TD error. This approach, referred to as the ``frequency-centric" approach, emphasizes rolling out these significant samples more frequently, leading to more updates. On the other hand, DTD($0$) follows a ``magnitude-centric" approach, increasing the magnitude of each update to directly address the imbalanced frequency. The following proposition demonstrates the connection between DTD($0$) and PER:
\begin{proposition}
  For a Markovian dataset $\mathcal{D}$ generated from $\pi$, of a size $N$, then:
  \begin{equation}
    \label{per:1}
    \begin{aligned}    
      & \mathbb{E}_{\text{uniform}}\left[f^{2}(s) \left( v^{\text{target}} - \hat{v}(s, \bt) \right)^{2}\right] \\
      & = c \cdot \mathbb{E}_{q}\left[\left(v^{\text{target}} - \hat{v}(s, \bt) \right)^{2}\right],
    \end{aligned}
  \end{equation}
  where
  \begin{equation}
    \label{per:2}
    \begin{aligned}
      q(s) =  \frac{f^{2}(s)}{\sum\limits_{s \in \mathcal{D}} f^{2}(s')}\quad  c = \frac{\sum\limits_{s \in \mathcal{D}} f^{2}(s')}{N}.
    \end{aligned}
  \end{equation}  
\end{proposition}
What this conveys is that sampling from a priority distribution $q(s)$, scaled by a constant factor $c$, is analogous to uniform sampling with an emphasized objective. It is worth noting that any priority distribution of interest can be obtained by taking the square root of the corresponding emphasis function. This equivalence between DTD($0$) and PER is compelling, as it directly integrates the emphasis into the objective to achieve the same prioritization effect.

However, we refrain from specifying a fixed form for the emphasis function, as it should ideally be tailored to the specifics of each problem, considering factors like problem complexity and size. In practical applications, employing function approximations to extend the emphasis across similar states could prove useful. However, delving into the details of the estimation method for such function approximations lies beyond the scope of this study and presents an intriguing avenue for future research.

\section{Discussions on Possible Forms}
In this section, we open a dialogue on designing emphasis functions that would maximize the effectiveness of our approach.

The emphasis can be future-predicting, such as selecting the conditional entropy $f_{t + n} = H(G | A_{\leq t + n - 1}, S_{\leq t + n})$. This choice will prioritize the $n$-step return with maximal information contained in $(A_{t + n - 1}, S_{t + n})$ about the return $G$. To provide a more intuitive understanding, for $\lambda = 1$, the quantity $f_{t + n - 1} - f_{t + n}$ equates to the mutual information $\mathcal{I}(G | A_{\leq t+ n - 2}, S_{\leq t + n - 1}; A_{t + n - 1}, S_{t + n})$.

The emphasis can also be history-summarizing, such as choosing the negative exponential of variance of $n$-step return $f_{t + n} = \exp{(-\mathbb{V} [G_{t}^{(n)}])}$. This approach would resemble the experiments involving perturbed rewards, allowing the distinction of various return functions based on their noise levels. Such an approach could be beneficial for model-based planning, as the accumulation of errors in the model can render predictions less reliable \cite{DBLP:conf/nips/JannerFZL19}.

It is also intriguing to assess the expected immediate reward for each state, with which it becomes possible to categorize the state space based on higher rewards. This enables the allocation of more resources towards predicting the value function of these valuable states, enhancing their utility in control tasks.

\section{Related Work}
\label{sec:related-work}
In the context of TD learning, ETD \cite{DBLP:journals/jmlr/SuttonMW16} employs a follow-on trace coupled with an interest function to address the stability issue in the off-policy TD learning. PTD \cite{DBLP:conf/icml/AnandP21} introduces a preference function that is reversely related to $\lambda$, enabling interpolation between TD(0) update and no update to handle partial observation challenges. Additionally, TD($\lambda, w(\cdot)$) \cite{DBLP:journals/corr/abs-2202-09699} proposes a selective eligibility trace for reweighting historical states, similar to our approach. However, it disregards the consideration of the relative influence of propagating the TD error.

Emphasizing the significance of certain states is a recurring concept in various domains. \cite{DBLP:conf/nips/McLeodLSJKWW21} addresses the multi-prediction problem with a GVF \cite{DBLP:conf/atal/SuttonMDDPWP11} by focusing on learning a subset of states for each prediction, facilitated by an underlying interest function. \cite{DBLP:conf/nips/ImaniGW18} introduces an extension of emphatic weighting into the domain of control, resulting in an off-policy emphatic policy gradient that incorporates a state-dependent interest function. However, the process of adapting or selecting an appropriate interest function can be challenging. In response, \cite{DBLP:conf/nips/KlissarovFMAKS22} proposes a meta-gradient to dynamically adjust the interest, highlighting the advantages of identifying crucial states, thereby enhancing the efficacy of transfer learning across RL tasks. It is possible to combine our method with these techniques. The intriguing question, however, is how our approach can be most suitable for control problems. The notion of selective updating also finds application in option learning, manifesting either through initiation sets \cite{DBLP:journals/ai/SuttonPS99}, or via the utilization of an interest function \cite{DBLP:conf/aaai/KhetarpalKCBP20}. In model-based RL, \cite{DBLP:conf/icml/AbbasSTW20} combines the learned variance to adjust the weighting of targets derived from model planning to account for the limited model capacity, and \cite{DBLP:conf/nips/BuckmanHTBL18} leverages a bias-variance trade-off to determine a weighting scheme.

Regarding the prioritized sampling, PER \cite{DBLP:journals/corr/SchaulQAS15} addresses the initial step of considering the significance of different samples. In terms of the expected gradient perspective, \nocite{DBLP:conf/nips/FujimotoMP20} shows the correlation between an $l^{1}$ loss employing a prioritized sampling scheme and the uniformly sampled MSE loss. Additionally, \cite{DBLP:conf/uai/PanMFWYR022} establishes an alternative equivalence between the uniformly sampled cubic loss and the prioritized MSE loss.

\section{Conclusions}
\label{sec:conclusions}
In this paper, we introduced an emphasis-aware TD learning approach that takes into account the importance of historical states and the relative significance of propagating TD errors. Our method offers enhanced flexibility in selecting the emphasis. From various angles, we demonstrated its efficacy in challenging scenarios involving visitation imbalance and outcome noise. It not only restores balance to updates but also distinguishes between different noise levels, leading to improved predictions. We explored adaptive emphasis and confirmed its effectiveness in accelerating learning. Theoretical analysis established a contraction condition for algorithm convergence, offering practical insights into selecting the emphasis function. We also presented insights into extensions for DRL, including the proposed DAE and an equivalence between DTD(0) and PER. Additionally, we discussed potential forms of emphasis, which could be valuable when integrating with function approximations for future work.

\bibliography{main.bib}
\end{document}